\documentclass[pmlr]{jmlr} 

\usepackage{booktabs} 
\usepackage{graphicx} 
\usepackage{url}      
\usepackage{amsmath}  
\usepackage{algorithm2e} 

\usepackage{bm}

\usepackage{mathtools}
\usepackage{tikz}
\usepackage{nicefrac}

\renewcommand{\arraystretch}{1.3}

\newcommand{\class}{c}
\newcommand{\Class}{C}
\newcommand{\Classes}{\mathcal{C}}

\newcommand{\prediction}{\hat{\class}}

\newcommand{\predictedset}[1]{\hat{\Classes}_{#1,\features}}

\newcommand{\classifier}{h}
\newcommand{\classifierp}{h_{\proba}}
\newcommand{\classifierpa}{h_{\proba}}

\newcommand{\feature}[1]{f_{#1}}
\newcommand{\Feature}[1]{F_{#1}}
\newcommand{\FeatureSet}[1]{\mathcal{F}_{#1}}
\newcommand{\features}{\boldsymbol{f}}
\newcommand{\Features}{F}
\newcommand{\FeaturesSet}{\boldsymbol{\mathcal{F}}}
\newcommand{\numfeatures}{N}

\newcommand{\proba}{p}

\newcommand{\pertproba}{p'}
\newcommand{\probfamily}{\mathcal{P}}
\newcommand{\classprobfamily}{\probfamily_{\Class}}
\newcommand{\featprobfamily}{\probfamily_{\Feature{i} \vert \class}}
\newcommand{\featprobfamilypr}{\probfamily_{\Feature{i} \vert \prediction}}

\newcommand{\expectation}{E}
\newcommand{\prev}[2]{\expectation_{#1}\left( #2 \right)}
\newcommand{\prevbr}[2]{\expectation_{#1}\big[ #2 \big]}

\newcommand{\lowprev}[2]{\underline{\expectation}_{#1}\left( #2 \right)}
\newcommand{\lowprevbr}[2]{\underline{\expectation}_{#1}\big[ #2 \big]}
\newcommand{\indicator}[1]{\mathbb{I}_{#1}}
\newcommand{\probset}[1]{\Sigma_{#1}}
\newcommand{\classifprobset}{\probset{\Classes \times \FeaturesSet}}

\newcommand{\nbcpert}{\probfamily_{\mathrm{NBC}}}
\newcommand{\nbcparampert}{\probfamily^{\delta}_{\mathrm{NBC}}}

\newcommand{\parampert}[1]{\probfamily^{#1}}

\newcommand{\node}{v}
\newcommand{\leaf}{l}
\newcommand{\childnode}{u}
\newcommand{\rootnode}{r}
\newcommand{\childrennode}[1]{\mathrm{ch}(#1)}
\newcommand{\descnode}[1]{\mathrm{de}(#1)}
\newcommand{\descnoleafnode}[1]{\mathrm{sde}(#1)}
\newcommand{\weightsnode}{w}
\newcommand{\weightstree}{\overline{\weightsnode}}
\newcommand{\pertweights}{\mathcal{W}}
\newcommand{\pertweightsvec}[1]{\pertweights_{#1}}
\newcommand{\pertweightstree}[1]{\overline{\pertweights}_{#1}}

\newcommand{\vfunc}[1]{\underline{V}(#1)}

\newcommand{\rob}{r}

\definecolor{darkgreen}{rgb}{0.01, 0.65, 0.18}
\definecolor{rob_blue}{rgb}{0.0, 0.75, 1.0}
\definecolor{unc_yellow}{rgb}{0.95, 0.7, 0.15}
\definecolor{hybrid_green}{rgb}{0.1, 0.7, 0.2}

\jmlrvolume{TBD}
\jmlryear{2026}
\jmlrworkshop{Probabilistic Graphical Models (PGM)}

\title[Local Robusness Quantification for NBC and GeFs: a General Approach]{Local Robustness Quantification for Naive Bayes Classifiers and Generative Forests: a General Approach}

\author{\Name{Adrián Detavernier} \Email{adrian.detavernier@ugent.be}\\
  \Name{Jasper De~Bock} \Email{jasper.debock@ugent.be}\\
  \addr Foundations Lab for imprecise probabilities\\
  Ghent University\\
  Ghent, Belgium
}

\editor{Gustau Camps-Valls, Manuele Leonelli and Gherardo Varando}

\begin{document}

\maketitle

\begin{abstract}
We provide methods for calculating the robustness of the predictions of two types of generative classifiers whose underlying distribution is a Probabilistic Graphical Model (PGM): naive Bayes classifiers and generative forests (a probabilistic extension of random forests).
Following the paradigm of robustness quantification, we define the robustness of a prediction as the extent to which the distribution of the classifier can be perturbed without changing this prediction.
We consider perturbations obtained by varying the local models of the PGMs within general neighborhoods and focus in particular on epsilon-contamination, total variation distance and chi-squared divergence balls.
We test our methods on benchmark datasets, demonstrate that the robustness value of a prediction serves as an indicator for its trustworthiness and compare our approach with other such indicators.
\end{abstract}

\begin{keywords}
trustworthy classification, generative forests, naive Bayes classifier, imprecise probabilities
\end{keywords}

\section{Introduction}
\label{sec:intro}
Say that you are writing a paper and the deadline is approaching, so you have to quickly finish the paper and submit it.
However, the central proof of the paper is not working yet, so, to save time, you decide to consult a Large Language Model (LLM) to write the proof for you.
You get the proof from the LLM, but you realize that time's up and that you are not able to check the proof before submitting the paper.
\footnote{Disclaimer: this situation is purely hypothetical and does not reflect the actual practices of the author(s).}
Would you trust the proof of the LLM?
Would you take the risk?
Or would you only do it if the LLM is extremely confident about the proof?
If such a (hopefully unrealistic) situation arises, it would be very useful if the LLM were to give us some indication of the trustworthiness of its output.
Similarly for other high-stakes settings, such as medical diagnosis, it would be very useful if a model were able to assess the trustworthiness of its own predictions such that the user could take this information into account when making decisions based on these predictions.

In this paper, we consider one possible approach that attempts to accompany the individual predictions of a classifier with an indication of their trustworthiness: Robustness Quantification (RQ) \citep{detavernier2025robustness}.
RQ in particular quantifies the so-called \emph{robustness} of a classifier's predictions, which is the extent to which the model can be perturbed without changing the prediction, the idea being that predictions with high robustness are more trustworthy because they are less sensitive to errors  made while training the model.
RQ has so far been studied for classifiers based on several types of Probabilistic Graphical Models (PGMs): the Naive Bayes Classifier (NBC) \citep{detavernier2025robustness}, for Generative Forests (GeFs) \citep{correia2020robustclassificationdeepgenerative}, and for Bayesian networks and Markov random fields \citep{NIPS2014_09662890}.
These results have furthermore shown that the robustness of a prediction, defined in this way, can be a useful indicator for its trustworthiness, in the sense that predictions with a higher robustness value tend to be correct more often than predictions with a lower robustness value.

In this work, we show how RQ can be reduced to an optimization problem.
Additionally, we use ideas from imprecise probability theory~\citep{augustin2014introduction} and credal classification \citep{ITIPclassification} to provide general methods for reducing this optimization problem to local optimizations if the PGM is an NBC or GeF that is perturbed by local perturbations of its local probability or weight functions, respectively. In doing so, we generalize existing methods that focussed on \(\varepsilon\)-contamination perturbations.
We illustrate this generality by considering as perturbations total variation distance and chi-squared divergence balls, and we show in our experiments that the resulting robustness values still correlate with accuracy, and moreover that they are competitive with those obtained with \(\varepsilon\)-contamination and with indicators of trustworthiness that are based on uncertainty.

\section{Classification: Setting and Notation}
\label{sec:classification}
To formalize the setting, namely classification with discrete features, we introduce some notation.
The goal of a classifier is to predict the class \(\class\) of an instance, which is an element of a finite set of classes \(\Classes\).
When the class is unknown, we refer to it as the class variable \(\Class\).
To make this prediction, we typically have some information about the instance in the form of a set of features \(\features\).
Each of the features is denoted by a variable \(\Feature{i}\) and takes values \(\feature{i}\) in a finite set \(\FeatureSet{i}\), with \(i \in \{1, \dots, \numfeatures\}\), where \(\numfeatures\) is the number of features.
We will collect all the features of an instance in a vector, and refer to this vector as the feature vector or just the features of the instance.
The feature vector of an instance is given by \(\features = (\feature{1}, \dots, \feature{\numfeatures})\), where \(\feature{i}\) is the value of the \(i\)-th feature for this instance.
All such possible feature vectors are collected in the set \(\FeaturesSet = \FeatureSet{1} \times \dots \times \FeatureSet{\numfeatures}\).
Every instance is thus determined by the combination of a class \(\class\) and its features \(\features\).
A classifier \(\classifier: \FeaturesSet \to \Classes\) can therefore be seen as a function that maps feature vectors to classes.
In the ideal scenario, the predicted class \(\prediction \coloneq \classifier(\features)\) is the same as the true class of the instance.
In practice, however, this is not always the case.
That is exactly why assessing the trustworthiness of the predictions of a classifier is important.

This contribution focusses on generative classifiers that first estimate a joint probability mass function \(\proba\) on \(\Classes \times \FeaturesSet\), and then predict the class \(\prediction\) that maximizes the conditional probability  \(\proba(\class\vert\features) = \nicefrac{\proba(\class, \features)}{\proba(\features)}\), with \(\proba(\features) \coloneq \sum_{\class' \in \Classes} \proba(\class', \features)\).
This \(\prediction\) may not be unique though, or \(\proba(\class\vert\features)\) might be ill-defined if \(\proba(\features) = 0\), which is why formally, given a set of features \(\features\), a generative classifier \(\classifierp:\FeaturesSet \to \Classes\) chooses the predicted class \(\prediction\) from the set
\begin{equation}
    \label{eq:prob_classifier}
    \predictedset{\proba} \coloneq
    \left.
    \begin{cases}
        \arg \max_{\class \in \Classes} \proba(\class\vert\features) &\text{if }\proba(\features) > 0\\
        \Classes &\text{if } \proba(\features) = 0
    \end{cases}
    \right\}
    = \arg \max_{\class \in \Classes} \proba(\class, \features).
\end{equation}
In practice, \(\predictedset{\proba}\) will of course typically be a singleton, in which case \(\prediction\) is completely determined by \(\proba\), but our methods also work for features for which this is not the case.

\section{Robustness of a Prediction}\label{sec:robustness}

The distribution \(\proba\) learned by a probabilistic classifier is typically not the true data generating distribution; data subject to distribution shift, the use of a specific model architecture, or randomness in the training process/model parameters can all lead to deviations of \(\proba\) from the true distribution, and hence to a suboptimal prediction.
So, what if we would slightly change the distribution \(\proba\) that the model learned?
Would the prediction change?
How much would we have to change the distribution before the prediction changes?
RQ answers this last question by quantifying how much the learned distribution \(\proba\) of the model can be changed without changing the resulting prediction of the classifier.
To formalize this idea, we consider perturbations of the mass function \(\proba\), which are simply sets of mass functions that contain \(\proba\).

Let \(\Omega\) be a finite set of outcomes and let \(\probset{\Omega}\) be the set of all probability mass functions on \(\Omega\): maps from \(\Omega\) to \([0,1]\) that sum up to one.
\begin{definition}
    \label{def:perturbation}
	Consider a mass function \(\proba\in\probset{\Omega}\) on a finite set of outcomes \(\Omega\).
    A \emph{\textbf{perturbation}} of \(\proba\) is a compact set \(\probfamily\subseteq \probset{\Omega}\) of mass functions on \(\Omega\) such that \(\proba \in \probfamily\).
\end{definition}
This is a general and deliberately vague definition since it does not specify how the perturbation is constructed.
In practice, we will usually construct such a perturbation as a neighborhood around \(\proba\), with \(\proba\in \probset{\Classes\times\FeaturesSet}\) a mass function learned by a probabilistic classifier.

Given a perturbation \(\probfamily\) of \(\proba\) and set of features \(\features\), we can now construct the set of possible predictions \(\predictedset{\probfamily}\coloneq \bigcup_{\pertproba \in \probfamily} \predictedset{\pertproba}\) that are compatible with at least one of the perturbed models \(\proba'\in\probfamily\).
The prediction \(\prediction\) is said to be robust w.r.t. the perturbation \(\probfamily\) if the only such compatible prediction is \(\prediction\) itself.
\begin{definition}
	\label{def:robustness}
	Let \(\classifierpa\) be a generative classifier corresponding to a mass function \(\proba\).
	Let \(\prediction\) be the prediction according to \(\classifierpa\) for the set of features \(\features\), let \(\probfamily\) be a perturbation of \(\proba\), and let \(\predictedset{\probfamily}\) be the set of possible predictions w.r.t. the perturbation \(\probfamily\).
	Then \(\prediction\) is \emph{\textbf{robust}} w.r.t. the perturbation \(\probfamily\) if \(\predictedset{\probfamily}=\{\prediction\}\).
\end{definition}
For readers that are familiar with imprecise probabilities and credal classification, these ideas should be familiar. In that context, our perturbations would typically be called credal sets and \(\predictedset{\probfamily}\) would be an example of a corresponding credal classifier: a classification procedure that outputs a set of classes.\footnote{Many different imprecise decision criteria can be used to associate such a set of classes with a credal set~\cite{Troffaes_2007}. Our approach, which simply gathers the predictions of the different distributions in the credal set, essentially corresponds to the use of E-admissibility; it has the advantage of having an intuitive sensitivity-analysis interpretation. Credal classification instead typically uses another decision criterion, called maximality, for which the corresponding set of classes can more easily be determined using optimization techniques. For the purposes of determining robustness, however, this choice does not matter. For readers familiar with maximality, this should follow easily from~Theorem~\ref{def:robustness} further on.} Robustness of an instance then corresponds to this credal classifier being determinate, in the sense that its output contains only a single class. An important observation in credal classification, which inspired the development of RQ~\citep{NIPS2014_09662890}, is that traditional classifiers tend to perform worse on  instances where their credal counterparts remain indeterminate~\citep{corani2008learning}.

The ideas behind RQ are similar, but it has different aims. First, RQ is not interested in determining \(\predictedset{\probfamily}\), but focusses solely on robustness (determinacy). Second, instead of determining whether a prediction is robust w.r.t. a given fixed perturbation, RQ instead wants to quantify \emph{how} robust it is.
RQ does this by controlling the size of the perturbation in a parametrized manner and increasing this size until the prediction of the model is no longer robust, or thus until at least one distribution in the neighborhood predicts a different class.
This smallest perturbation size for which the prediction is no longer robust is then used as a numeric measure of robustness.
To formalize this, we consider \emph{parametrized perturbations}.
\begin{definition}\label{def:param_pert}
	Consider a mass function \(\proba\in\probset{\Omega}\) on a finite set of outcomes \(\Omega\).
    Let \(\parampert{\delta}\) be a perturbation of \(\proba\) for all \(\delta \in \Delta \subseteq \mathbb{R}_{\geq 0}\), where \(\Delta\) is a set of possible values for \(\delta\) that includes zero.
    Then the family \(\parampert{\Delta} \coloneq \left(\parampert{\delta}\right)_{\delta \in \Delta}\) is called a \emph{\textbf{parametrized perturbation}} of \(\proba\) if the following conditions hold: (1) if \(\delta = 0\), then \(\parampert{0} = \{\proba\}\); and (2) if \(\delta_1 \leq \delta_2\), then \(\parampert{\delta_1} \subseteq \parampert{\delta_2}\).
\end{definition}
Clearly, the bigger the value of \(\delta\), the bigger the perturbation, in the sense that it contains more distributions.
Concrete examples of such parametrized perturbations are given in Section~\ref{sec:local_pert}.
Next, we formalize the numeric measure of robustness.
\begin{definition}\label{def:robustness_value}
	Let \(\classifierpa\) be a generative classifier corresponding to a mass function \(\proba\), and let \(\prediction\) be the prediction according to \(\classifierpa\) for the set of features \(\features\).
	Let \(\parampert{\Delta}\) be a parametrized perturbation of \(\proba\), with \(\Delta \subseteq \mathbb{R}_{\geq 0}\) the possible values for the parameter \(\delta\).
    Then the \emph{\textbf{robustness (value)}} \(\rob_{\parampert{\Delta}}(\features)\) w.r.t. \(\parampert{\Delta}\) of an instance with features \(\features\) is the infimum \(\delta \in \Delta\) for which the prediction \(\prediction\) is no longer robust w.r.t. the perturbation \(\parampert{\delta}\).
    If there is no such \(\delta\), then we set \(\rob_{\parampert{\Delta}}(\features) = +\infty\).
\end{definition}
Henceforth, whenever we use the notation \(\parampert{\delta}\) for a perturbation, we implicitly assume that it is part of a parametrized perturbation \(\parampert{\Delta}\) for some set \(\Delta\) of possible values for \(\delta\).

This definition faithfully captures the ideas behind RQ, but it isn't the most practical one because, at first sight, checking whether a prediction is robust w.r.t. a perturbation would require us to determine the predictions of all the distributions in the perturbation.
To address this issue, we now proceed to provide an equivalent, more practical, characterization of robustness.
This characterization is expressed in terms of the lower expectation of a perturbation \(\probfamily\), defined for any real-valued function \(f\) on \(\Omega\) by
\begin{equation}
    \label{eq:def_lower_expectation}
    \lowprev{\probfamily}{f} \coloneq \min_{\pertproba \in \probfamily} \prev{\pertproba}{f},
\end{equation}
where \(\prev{\pertproba}{f}=\sum_{x\in\Omega} \pertproba(x)f(x)\) is the expectation of \(f\) w.r.t. the mass function \(\pertproba\).
If \(\Omega = \Classes\times\mathcal{Y}\) for some set \(\mathcal{Y}\), then for any \(\class\in\Classes\),  we use \(\indicator{\class}\) to denote the function that takes the value 1 if \(x\in\{\class\}\times\mathcal{Y}\) and 0 otherwise.
If \(\Omega = \FeaturesSet\times\mathcal{Y}\), \(\indicator{\features}\) is defined similarly.
\begin{theorem}
	\label{th:def_robustness}
	Let \(\classifierpa\) be a generative classifier corresponding to a mass function \(\proba\), let \(\prediction\) be the prediction according to \(\classifierpa\) for the set of features \(\features\), and let \(\probfamily\) be a perturbation of \(\proba\).
	Then \(\prediction\) is \emph{robust} w.r.t. the perturbation \(\probfamily\) if and only if
    \begin{equation}
        \label{eq:def_rob_probs}
		\forall \pertproba \in \probfamily, \forall \class \in \Classes \backslash \{\prediction\} \colon \pertproba(\prediction, \features) - \pertproba(\class, \features) > 0,
    \end{equation}
    or, equivalently, if
	\begin{equation}
		\label{eq:def_rob_lower}
		\min_{\class \in \Classes \backslash \{\prediction\}} \lowprevbr{\probfamily}{(\indicator{\prediction}-\indicator{\class})\indicator{\features}} > 0.
	\end{equation}
\end{theorem}
\begin{proof}
	By Definition~\ref{def:robustness} and Equation~\eqref{eq:prob_classifier}, \(\prediction\) is robust w.r.t. \(\probfamily\) if and only if \(\predictedset{\proba'}=\{\prediction\}\) for all \(\pertproba \in \probfamily\).
	This holds if and only if
    \(
		\forall \pertproba \in \probfamily, \forall \class \in \Classes \backslash \{\prediction\} \colon \pertproba(\prediction, \features) - \pertproba(\class, \features) > 0,
    \)
    which proves the first equivalence of the theorem.
    Since \(\probfamily\) is compact, the previous condition is equivalent to requiring that the minimum of this difference over all \(\class \in \Classes \backslash \{\prediction\}\) and all \(\pertproba \in \probfamily\) is positive:
    \(
        \min_{\class \in \Classes \backslash \{\prediction\}} \min_{\pertproba \in \probfamily} \pertproba(\prediction, \features) - \pertproba(\class, \features)  > 0.
    \)
    Since \(\pertproba(\prediction, \features) - \pertproba(\class, \features) = \prevbr{\pertproba}{\indicator{\prediction}\indicator{\features}-\indicator{\class}\indicator{\features}}\), this is equivalent to
    \(
        \min_{\class \in \Classes \backslash \{\prediction\}} \lowprevbr{\probfamily}{(\indicator{\prediction}-\indicator{\class})\indicator{\features}} > 0,
    \)
    as claimed.
\end{proof}
Consequently, checking whether a prediction is robust w.r.t. a perturbation \(\probfamily\) only requires us to check Equation~\eqref{eq:def_rob_lower} of Theorem~\ref{th:def_robustness}.
Furthermore, note that the function \(R:\Delta\to\mathbb{R}:\delta \mapsto \min_{\class \in \Classes \backslash \{\prediction\}} \lowprev{\parampert{\delta}}{(\indicator{\prediction}-\indicator{\class})\indicator{\features}}\) is non-increasing, since the bigger the perturbation, the smaller the lower expectation.
Finding the robustness value therefore amounts to finding the infimum \(\delta\) for which \(R(\delta)=0\); that is, we need to find the smallest root of a non-increasing function.
In practice, the function \(R\) will often be continuous and strictly decreasing, which makes this root-finding problem even easier to solve.
Several methods exist for finding roots of such functions efficiently, which make the quantification of robustness feasible in practice.

Now that we have the theoretical tools to quantify the robustness of predictions of classifiers, we will apply these tools to two specific types of generative classifiers.

\section{Naive Bayes Classifier}
The first type of generative classifier we will be looking at is the Naive Bayes Classifier (NBC).
This is a very simple model that is often used as a good baseline for more complex classifiers.
The central assumption on which the Naive Bayes Classifier is based is that the features are conditionally independent given the class.
This implies that there is a mass function \(\proba_{\Class} \in \probset{\Classes}\) and, for all \(\class \in \Classes\) and \(i \in \{1, \dots, \numfeatures\}\), a mass function \(\proba_{\Feature{i}\vert\class} \in \probset{\FeatureSet{i}}\) such that
\begin{equation}
	\label{eq:NB_joint}
	\proba(\class, \features) = \proba_{\Class}(\class) \prod_{i=1}^{\numfeatures} \proba_{\Feature{i}\vert\class}(\feature{i}), \quad \text{for all } \class \in \Classes \text{ and } \features \in \FeaturesSet.
\end{equation}
We now consider a specific parametrized perturbation for such NBCs; it is a parametrized version of the model of the Naive Credal Classifier~\citep{ZAFFALON20025}---a generalization of an NBC based on imprecise probabilities.
In particular, we consider a local perturbation \(\classprobfamily^{\delta}\) of \(\proba_{\Class}\) and local perturbations \(\featprobfamily^\delta\) of \(\proba_{\Feature{i}\vert\class}\) for all \(\class \in \Classes\) and all \(i \in \{1, \dots, \numfeatures\}\), and let
\begin{equation}\label{eq:nbcperturbation}
	\nbcpert^{\delta} \coloneq \Big\{ \pertproba \in \Sigma_{\Classes \times \FeaturesSet} \colon \pertproba_{\Class} \in \classprobfamily^{\delta},\ \pertproba_{\Feature{i}\vert\class} \in \featprobfamily^\delta,\ \proba'(\class, \features) = \proba_{\Class}'(\class) \prod_{i=1}^{\numfeatures} \proba_{\Feature{i}\vert\class}'(\feature{i}) \Big\}
\end{equation}
be the corresponding perturbation of \(\proba\).
This is simply the set of all NBCs whose local models are taken from \(\classprobfamily^{\delta}\) and \(\featprobfamily^\delta\).
The superscript \(\delta\) indicates that these perturbations are part of a parametrized perturbation where the parameter \(\delta\) controls the size of the perturbations, enabling us to associate a robustness value with each instance.

Since the local models of the NBC are independent, and all models in the perturbation \(\nbcparampert\) are also NBCs, we can use this structure to simplify the condition for robustness given by Theorem~\ref{th:def_robustness} even further.

\begin{theorem}
    \label{th:NBC_robustness}
	Let \(\classifierpa\) be a Naive Bayes classifier with local models \(\proba_{\Class}\) and \(\proba_{\Feature{i}\vert\class}\) for all \(\class \in \Classes\) and all \(i \in \{1, \dots, \numfeatures\}\) and corresponding joint mass function \(\proba\) defined by Equation~\eqref{eq:NB_joint}.
    Let \(\prediction\) be the prediction according to \(\classifierpa\) for the set of features \(\features\).
    Consider local perturbations \(\classprobfamily^{\delta}\) and \(\featprobfamily^\delta\) for all \(\class \in \Classes\) and all \(i \in \{1, \dots, \numfeatures\}\).
    Then \(\prediction\) is robust w.r.t. the corresponding perturbation \(\nbcpert^{\delta}\) of \(\proba\) if and only if
    \begin{equation}
        \label{eq:th_NBC_robustness}
        \min_{\class \in \Classes \backslash \{\prediction\}} \lowprev{\classprobfamily^{\delta}}{\indicator{\prediction} \prod_{i=1}^{\numfeatures} \underline{\proba}_{\Feature{i}\vert\prediction}(\feature{i}) - \indicator{\class} \prod_{i=1}^{\numfeatures}\overline{\proba}_{\Feature{i}\vert\class}(\feature{i}) } > 0,
    \end{equation}
    where \(\underline{\proba}_{\Feature{i}\vert\prediction}(\feature{i}) = \min_{\pertproba_{\Feature{i}\vert\prediction} \in \featprobfamilypr^\delta} \pertproba_{\Feature{i}\vert\prediction}(\feature{i})\) and \(\overline{\proba}_{\Feature{i}\vert\class}(\feature{i}) = \max_{\pertproba_{\Feature{i}\vert\class} \in \featprobfamily^\delta} \pertproba_{\Feature{i}\vert\class}(\feature{i})\).
\end{theorem}

\begin{proof}
    By Theorem~\ref{th:def_robustness} \(\prediction\) is robust w.r.t. \(\nbcpert^{\delta}\) if and only if
    \begin{equation}
         \min_{\class \in \Classes \backslash \{\prediction\}} \min_{\pertproba \in \nbcpert^{\delta}} \pertproba(\prediction, \features) - \pertproba(\class, \features) > 0.
    \end{equation}
    Due to Equation~\eqref{eq:NB_joint} and the structure of \(\nbcpert^{\delta}\) in Equation~\eqref{eq:nbcperturbation}, this is equivalent to
    \begin{equation}
         \min_{\class \in \Classes \backslash \{\prediction\}} \min_{\pertproba_{\Class} \in \classprobfamily^{\delta}} \min_{\forall \class\in\Classes,\forall i\in\{1,\dots,\numfeatures\}:\ \pertproba_{\Feature{i}\vert\class} \in \featprobfamily^\delta}  \pertproba_{\Class}(\prediction) \prod_{i=1}^{\numfeatures} \pertproba_{\Feature{i}\vert\prediction}(\feature{i}) -  \pertproba_{\Class}(\class) \prod_{i=1}^{\numfeatures} \pertproba_{\Feature{i}\vert\class}(\feature{i}) > 0.
    \end{equation}
    To minimize this difference, we want to minimize the first (positive) term and maximize the negation of the second (negative) term.
    Because of the independence of the local mass functions in the above expression, we can rewrite it as
    \begin{equation}
        \min_{\class \in \Classes \backslash \{\prediction\}} \min_{\pertproba_{\Class} \in \classprobfamily^{\delta}} \left(  \pertproba_{\Class}(\prediction) \prod_{i=1}^{\numfeatures} \min_{\pertproba_{\Feature{i}\vert\prediction} \in \featprobfamilypr^\delta} \pertproba_{\Feature{i}\vert\prediction}(\feature{i}\vert\prediction) - \pertproba_{\Class}(\class) \prod_{i=1}^{\numfeatures} \max_{\pertproba_{\Feature{i}\vert\class} \in \featprobfamily^\delta} \pertproba_{\Feature{i}\vert\class}(\feature{i}\vert\class) \right) > 0.
    \end{equation}
    By using the definitions of the local lower and upper probabilities and local lower expectations, the above condition is equivalent to
    \begin{equation}
        \min_{\class \in \Classes \backslash \{\prediction\}} \lowprev{\classprobfamily^{\delta}}{\indicator{\prediction} \prod_{i=1}^{\numfeatures}\underline{\proba}_{\Feature{i}\vert\prediction}(\feature{i}) - \indicator{\class} \prod_{i=1}^{\numfeatures}\overline{\proba}_{\Feature{i}\vert\class}(\feature{i}) } > 0,
    \end{equation}
    as claimed.
\end{proof}
So, for NBCs the robustness value of an instance can be found by finding the smallest value \(\delta\) for which the condition of Theorem~\ref{th:NBC_robustness} is not satisfied anymore.

\section{Generative Forests}
Even in times when deep learning models are dominating the field of machine learning, tree-based models are still very much relevant \citep{NEURIPS2022_0378c769}, while being more interpretable and easier to train.
For this reason, and because of their easy to work-with structure, we choose to also apply RQ to tree-based models.
To be able to apply our approach, we opt for Generative Forests (GeFs), that, in contrast to most tree-based models, are generative classifiers.
A Generative Forest is either an ensemble of Generative Decision Trees (GeDTs) \citep{NEURIPS2020_8396b14c} or a uniform mixture of GeDTs \citep{correia2020robustclassificationdeepgenerative}; a GeDT can be seen as generative version of a Decision Tree (DT).
Because the mixture variant has a single joint distribution, it is more suitable for our purposes than the ensemble one.
We will therefore focus on the mixture variant.

\begin{figure}
    \begin{tikzpicture}
        \node(nothing) at (-4,0) {};
        \node (titleGeDT) at (-2.5, 0.4) [color=green!90!black] {\bf\large GeDT};

        \node (root) at (0,0) [circle, draw, line width=1pt, inner sep=-1pt, fill=green, opacity=0.6] {\bf\LARGE+};
        \node (root) at (0,0) [circle, draw, line width=1pt, inner sep=-1pt, fill=none] {\bf\LARGE+};
        \node (roottext) at (root.north) [yshift=0.15cm] {\(\node\)};
        \node (sumnode) at (-1.5,-1.5) [circle, draw, line width=1pt, inner sep=-1pt] {\bf\LARGE+};
        \node (sumtext) at (sumnode.north) [yshift=0.15cm] {\(\childnode_1\)};
        \node (leaf1) at (1.5,-1.5) [circle, draw, line width=1pt, minimum width=1.5em] {};
        \node (leaf1text) at (leaf1.north) [yshift=0.15cm] {\(\childnode_2\)};
        \node (leaf2) at (0,-3) [circle, draw, line width=1pt, minimum width=1.5em] {};
        \node (leaf2text) at (leaf2.north) [yshift=0.15cm] {\(t\)};
        \node (leaf3) at (-3,-3) [circle, draw, line width=1pt, minimum width=1.5em] {};
        \node (leaf3text) at (leaf3.north) [yshift=0.15cm] {\(s\)};

        \draw[-, line width=1pt] (root) -- (sumnode);
        \draw[-, line width=1pt] (root) -- (leaf1);
        \draw[-, line width=1pt] (sumnode) -- (leaf2);
        \draw[-, line width=1pt] (sumnode) -- (leaf3);

        \node (weightrootright) at (root.east) [xshift=.52cm, yshift=-.6cm, rotate=-42] {\small\(\weightsnode_{\node, \childnode_2}\)};
        \node (weightrootright) at (root.west) [xshift=-.5cm, yshift=-.55cm, rotate=42] {\small\(\weightsnode_{\node, \childnode_1}\)};
        \node (weightsumnodeleft) at (sumnode.west) [xshift=-.5cm, yshift=-.55cm, rotate=42] {\small\(\weightsnode_{\childnode_1, s}\)};
        \node (weightsumnoderight) at (sumnode.east) [xshift=.5cm, yshift=-.55cm, rotate=-42] {\small\(\weightsnode_{\childnode_1, t}\)};

        \draw[color=green!90!black, opacity=0.6, rounded corners=15pt, line width=2pt] (-3.5, .7) -- (2, .7) -- (2, -3.5) -- (-3.5, -3.5) -- cycle;

        \node (titleGeF) at (5.8, 0.4) [color=darkgreen] {\bf\large GeF};

        \node (GeFroot) at (7.75, 0) [circle, draw, line width=1pt, inner sep=-1pt, fill=darkgreen, opacity=0.6] {\bf\LARGE+};
        \node (GeFroot) at (7.75, 0) [circle, draw, line width=1pt, inner sep=-1pt, fill=none] {\bf\LARGE+};
        \node (GeFroottext) at (GeFroot.north) [yshift=0.15cm] {\(\rootnode\)};

        \node (gefchild1) at (5.6,-2.9) [circle, draw, line width=1pt, inner sep=-1pt, fill=green, opacity=0.6] {\bf\LARGE+};
        \node (gefchild1) at (5.6,-2.9) [circle, draw, line width=1pt, inner sep=-1pt, fill=none] {\bf\LARGE+};
        \node (gefchild1text) at (gefchild1.north) [yshift=0.13cm] {\(\node\)};
        \draw[color=green!90!black, opacity=0.8, rounded corners=7pt, line width=1.5pt] (5.2,-2.2) -- (6.0,-2.2) -- (6.0,-3.3) -- (5.2,-3.3) -- cycle;

        \node (gefchild2) at (6.7,-2.9) [circle, draw, line width=1pt, inner sep=-1pt, fill=green, opacity=0.6] {\bf\LARGE+};
        \node (gefchild2) at (6.7,-2.9) [circle, draw, line width=1pt, inner sep=-1pt, fill=none] {\bf\LARGE+};
        \node (gefchild2text) at (gefchild2.north) [yshift=0.18cm] {\(\node'\)};
        \draw[color=green!90!black, opacity=0.8, rounded corners=7pt, line width=1.5pt] (6.3,-2.2) -- (7.1,-2.2) -- (7.1,-3.3) -- (6.3,-3.3) -- cycle;

        \node (gefchild3) at (8.8,-2.9) [circle, draw, line width=1pt, inner sep=-1pt, fill=green, opacity=0.6] {\bf\LARGE+};
        \node (gefchild3) at (8.8,-2.9) [circle, draw, line width=1pt, inner sep=-1pt, fill=none] {\bf\LARGE+};
        \node (gefchild3text) at (gefchild3.north) [yshift=0.18cm, xshift=0.08cm] {\(\node''\)};
        \draw[color=green!90!black, opacity=0.8, rounded corners=7pt, line width=1.5pt] (8.4,-2.2) -- (9.2,-2.2) -- (9.2,-3.3) -- (8.4,-3.3) -- cycle;

        \node (gefchild4) at (9.9,-2.9) [circle, draw, line width=1pt, inner sep=-1pt, fill=green, opacity=0.6] {\bf\LARGE+};
        \node (gefchild4) at (9.9,-2.9) [circle, draw, line width=1pt, inner sep=-1pt, fill=none] {\bf\LARGE+};
        \node (gefchild4text) at (gefchild4.north) [yshift=0.18cm] {\(\node'''\)};
        \draw[color=green!90!black, opacity=0.8, rounded corners=7pt, line width=1.5pt] (9.5,-2.2) -- (10.3,-2.2) -- (10.3,-3.3) -- (9.5,-3.3) -- cycle;

        \node (dots) at (7.75, -3) {\bf\LARGE\(\dots\)};

        \draw[-, line width=1pt] (GeFroot) -- (gefchild1);
        \draw[-, line width=1pt] (GeFroot) -- (gefchild2);
        \draw[-, line width=1pt] (GeFroot) -- (gefchild3);
        \draw[-, line width=1pt] (GeFroot) -- (gefchild4);

        \node (weightrootchild1) at (GeFroot.west) [xshift=-.85cm, yshift=-1.3cm, rotate=54] {\small\(\weightsnode_{\rootnode, \node}\)};
        \node (weightrootchild1) at (GeFroot.west) [xshift=-.33cm, yshift=-1.3cm, rotate=67] {\small\(\weightsnode_{\rootnode, \node'}\)};
        \node (weightrootchild1) at (GeFroot.east) [xshift=.37cm, yshift=-1.44cm, rotate=-65 ] {\small\(\weightsnode_{\rootnode, \node''}\)};
        \node (weightrootchild1) at (GeFroot.east) [xshift=.92cm, yshift=-1.4cm, rotate=-54] {\small\(\weightsnode_{\rootnode, \node'''}\)};

        \draw[color=darkgreen, rounded corners=15pt, line width=2pt] (5, .7) -- (10.5, .7) -- (10.5, -3.5) -- (5, -3.5) -- cycle;

        \draw[->, line width=1.5pt] (2.5,-1.5) -- (4.5,-1.5);

    \end{tikzpicture}
    \caption{Visualization of the structures of a GeDT and a GeF.}
    \label{fig:GeFs_fig}
\end{figure}
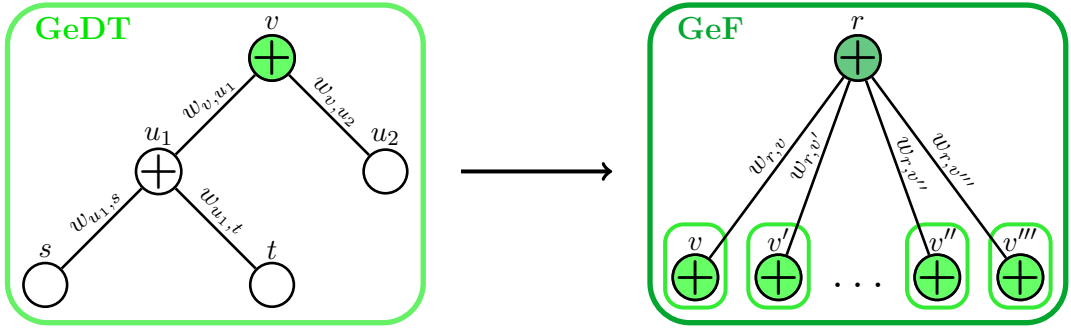

A GeDT can be seen as a very simple Probabilistic Circuit (PC) with, in our case, a binary tree structure containing only leaf and sum nodes.
A GeF combines several GeDTs into one PC, where the root node is a sum node of which all the children are the root of a GeDT, as shown in \figureref{fig:GeFs_fig}.
If \(\node\) is a node in a GeF, then its children are collected in \(\childrennode{\node}\), its descendants in \(\descnode{\node}\) (excluding \(\node\) itself), and its sum node descendants (excluding the leaf nodes) in \(\descnoleafnode{\node}\).
Note that each node in a such a tree is the root of a subtree containing all its descendants and itself.
If a node has no children, so \(\vert \childrennode{\node} \vert = 0\), then it is a leaf node; otherwise it is a sum node.
A sum node \(\node\) has for each of its children \(\childnode \in \childrennode{\node}\) a weight \(w_{\node, \childnode} \geq 0\) such that \(\sum_{\childnode \in \childrennode{\node}} w_{\node, \childnode} = 1\).
For the root node \(\rootnode\) we set all weights equal, so \(\weightsnode_{\rootnode, \node} = \frac{1}{\vert \childrennode{\rootnode}\vert}\) for all \(\node \in \childrennode{\rootnode}\).
The vector of all weights of a sum node \(\node\) is denoted by \(\weightsnode_{\node}\) and the collection of all weight vectors of the sum nodes in a GeF rooted at \(\node\) is denoted by \(\weightstree_{\node} \coloneq \{\weightsnode_{\node}\} \times \bigtimes_{\childnode \in \descnoleafnode{\node}} \{ \weightsnode_{\childnode} \}\).

A selling point of GeDTs and GeFs is that they can specify complex joint distributions \(\proba \in \probset{\Classes \times \FeaturesSet}\) in such a way that they can be easily computed.
They do so by associating simple joint distributions with each of their leaves and then using the sum nodes to take weighted averages of these simple distributions.
Following \citet{correia2020robustclassificationdeepgenerative}, for each leaf node \(\leaf\), its distribution \(\proba_{\leaf}\) is taken to be factorized over the class and features, meaning that \(\proba_{\leaf}(\class, \features) = \proba_{\leaf,\Class}(\class) \prod_{i=1}^{\numfeatures} \proba_{\leaf,\Feature{i}}(\feature{i})\) for all \(\class \in \Classes\) and \(\features \in \FeaturesSet\), where \(\proba_{\leaf,\Class}\) is a local mass function on \(\Classes\) and \(\proba_{\leaf,\Feature{i}}\) a local mass function on \(\FeatureSet{i}\) for all \(i \in \{1, \dots, \numfeatures\}\).
In the case of a sum node, its joint distribution \(\proba_{\node}\) is a weighted sum of the joint distributions of its children, given by \(\proba_{\node}(\class, \features) = \sum_{\childnode \in \childrennode{\node}} \weightsnode_{\node, \childnode} \proba_{\childnode}\) for any \(\class \in \Classes\) and \(\features \in \FeaturesSet\).
In this way, the collection of all weights \(\weightstree_{\node}\) of a GeDT or GeF rooted at \(\node\), together with the local models of the leaves, completely determine its joint distribution \(\proba_{\node}\).
To make this explicit, we denote the joint distribution of the tree rooted at \(\node\) by \(\proba_{\weightstree_{\node}}\), leaving the dependency on \(\proba_{\leaf,\Class}\) and \(\proba_{\leaf, \Feature{i}}\) implicit.
In particular, the joint distribution of a GeF with root \(\rootnode\) is denoted by \(\proba_{\weightstree_{\rootnode}}\).

The main goal of GeFs is to classify each instance based on its set of features \(\features\).
Let \(\rootnode\) be the root node of a GeF with joint distribution \(\proba_{\weightstree_{\rootnode}}\), then the predicted class \(\prediction = \classifier_{\proba_{\weightstree_{\rootnode}}}(\features)\) is chosen from \(\predictedset{\proba_{\weightstree_{\rootnode}}}\) as defined in Equation~\eqref{eq:prob_classifier}.
We now want to assess the robustness of this prediction by constructing neighborhoods of the learned distribution \(\proba_{\weightstree_{\rootnode}}\).
As we also did for the NBC, we do this by perturbing their local parameters.
Previous work on calculating the robustness of predictions of GeFs did this by perturbing the weights of the sum nodes using \(\varepsilon\)-contamination \citep{correia2020robustclassificationdeepgenerative}.
We adopt a similar approach but generalize it to allow for arbitrary local perturbations.
Note that, in contrast to the NBC, we won't be perturbing any probability distributions over classes or features directly here, but only indirectly by perturbing the weights of the sum nodes.

Given a sum node \(\node\) in a GeDT or GeF with weights \(\weightsnode_{\node}\), we consider a perturbation \(\pertweightsvec{\node}^{\delta}\) of \(\weightsnode_{\node}\) as defined in Definition~\ref{def:perturbation}.
In this work, we choose not to perturb the weights of the root node \(\rootnode\) of a GeF to keep the importance of each GeDT equal.
This means that we choose \(\pertweightsvec{\rootnode}^{\delta} = \{\weightsnode_{\rootnode}\}\) for all \(\delta\).
When we perturb the weights of all or multiple nodes in a GeF rooted at \(\rootnode\), then we are essentially perturbing \(\weightstree_{\rootnode}\).
The collection containing all possible combinations of perturbations of the weights \(\weightstree_{\rootnode}\) of the GeF rooted at \(\rootnode\) is given by \smash{\(\pertweightstree{\rootnode}^{\delta} \coloneq \{\weightsnode_{\rootnode}\} \times \bigtimes_{\node \in \descnoleafnode{\rootnode}} \pertweightsvec{\node}^{\delta}\)}.
Now, since each \(\weightstree_{\rootnode}' \in \pertweightstree{\rootnode}^{\delta}\) has a corresponding joint distribution \(\proba_{\weightstree_{\rootnode}'}\), each perturbation \(\pertweightstree{\rootnode}^{\delta}\) corresponds to a perturbation \(\smash{\probfamily_{\pertweightstree{\rootnode}^{\delta}}}\coloneq \big\{ \proba_{\weightstree_{\rootnode}'}: \weightstree_{\rootnode}' \in \pertweightstree{\rootnode}^{\delta} \big\}\) of \(\proba_{\weightstree_{\rootnode}}\).
The parameter \(\delta\) again indicates that this perturbation is part of a parametrized perturbation, enabling us to associate a robustness value with each instance.

By Theorem~\ref{def:robustness}, to calculate the robustness of a GeF, we need to calculate the lower expectation over a given perturbation of its joint probability mass function.
We show that this can be done efficiently by using the structure of GeFs.
To that end, we first define a (recursive) function of the nodes in a GeF and prove that the value of this function at the root node of a GeF is enough to check the robustness.

\begin{theorem}
    \label{th:GeF_robustness}
	Let \(\classifier_{\proba_{\weightstree_{\rootnode}}}\) be a GeF rooted at node \(\rootnode\) with mass function \(\proba_{\weightstree_{\rootnode}} \in \classifprobset\) corresponding to the weights \(\weightstree_{\rootnode}\).
    Let \(\prediction\) be the prediction according to \(\classifier_{\proba_{\weightstree_{\rootnode}}}\) for the set of features \(\features\).
    Consider local perturbations \(\pertweightsvec{\node}^{\delta}\) of the weights \(\weightsnode_{\node}\) of all sum nodes \(\node \neq \rootnode\) and let \smash{\(\probfamily_{\pertweightstree{\rootnode}^{\delta}}\)} be the corresponding perturbation of \(\proba_{\weightstree_{\rootnode}}\).
    Then \(\prediction\) is robust w.r.t. \(\probfamily_{\pertweightstree{\rootnode}^{\delta}}\) if and only if
    \begin{equation}
        \label{eq:th_GeF_robustness}
        \min_{\class \in \Classes \backslash \{\prediction\}} \vfunc{\rootnode} > 0,
    \end{equation}
    where the function \(\underline{V}\) is defined recursively by
    \begin{equation}
        \vfunc{\node} \coloneq \begin{cases}
            \left(\proba_{\leaf,\Class}(\prediction) - \proba_{\leaf,\Class}(\class)\right)\prod_{i=1}^{\numfeatures}\proba_{\leaf,\Feature{i}}(\feature{i}) & \text{if } \node \text{ is a leaf node } \leaf,\\
            \frac{1}{\vert \childrennode{r} \vert} \sum_{\childnode \in \childrennode{r}} \vfunc{\childnode} & \text{if } \node \text{ is the root node } r,\\
            \min_{\weightsnode_{\node}' \in \pertweightsvec{\node}^{\delta}} \sum_{\childnode \in \childrennode{\node}} \weightsnode_{\node, \childnode}' \vfunc{\childnode} & \text{if } \node \text{ is any other (sum) node}.
        \end{cases}
    \end{equation}
\end{theorem}

\begin{proof}
    Since \(\pertweightsvec{\rootnode}=\{\weightsnode_{\rootnode}\}\) in our case and the product that appears for the case $v=\ell$ is equal to \(E_{\proba_{\leaf}}((\indicator{\prediction}-\indicator{\class})\indicator{\features})\), it follows from Theorem~\ref{th:GeDT_lower} in Appendix~\ref{app:proofs} that \smash{\(\vfunc{\rootnode} = \lowprev{\probfamily_{\pertweightstree{\rootnode}^{\delta}}}{(\indicator{\prediction}-\indicator{\class})\indicator{\features}}\)}. The result now follows directly from Theorem~\ref{th:def_robustness}.
\end{proof}{}
Consequently, checking whether a prediction of a GeF is robust w.r.t. a perturbation can be done by recursively calculating the function \(\vfunc{\rootnode}\) at the root node of the GeF for each of the classes in \(\Classes \backslash \{\prediction\}\), and finding the robustness value of an instance amounts to finding the smallest value of \(\delta\) for which this is not the case.

\section{Local Perturbations}\label{sec:local_pert}
In this section, we discuss several examples of how a mass function can be perturbed, and give the formulas needed to calculate the robustness of predictions w.r.t. these perturbations.
Recall that the local perturbations that we consider in the paper are, in the case of the NBC, \(\classprobfamily^{\delta}\) of \(\proba_{\Class}\) and \(\featprobfamily^\delta\) of \(\proba_{\Feature{i}\vert\class}\) for all \(\class \in \Classes\) all \(i \in \{1, \dots, \numfeatures\}\); and in the case of GeFs, \smash{\(\pertweightsvec{\node}^{\delta}\)} of the weight vector \(\weightsnode_{\node}\) for each sum node \(\node\neq\rootnode\).
For the sake of ease of notation, we generically denote any of these as a perturbation \(\mathcal{Q}^{\delta}\) of a mass function \(q\) on a possibility space \(\Omega\).
Then for the three types of perturbations that we consider in this work, namely \(\varepsilon\)-contamination, Total Variation (TV) distance balls and \(\chi^2\)-divergence balls \citep{distancesref}, the perturbations are defined as follows:
\begin{itemize}
    \item \textbf{\(\varepsilon\)-contamination:} \(\mathcal{Q}^{\delta}_{\varepsilon} = \{ (1-\delta)q + \delta q' : q' \in \probset{\Omega} \}\), with \(\delta \in [0,1]\);
    \item \textbf{TV distance:} \(\mathcal{Q}^{\delta}_{\mathrm{TV}} = \{ q' \in \probset{\Omega} : \frac{1}{2} \sum_{\omega \in \Omega} \vert q(\omega) - q'(\omega) \vert \leq \delta \}\), with \(\delta \in [0,1]\);
    \item \textbf{\(\chi^2\)-divergence:} \(\mathcal{Q}^{\delta}_{\chi^2} = \{ q' \in \probset{\Omega} : \sum_{\omega \in \Omega} \frac{(q'(\omega) - q(\omega))^2}{q(\omega)} \leq \delta \}\), with \(\delta \geq 0\).
\end{itemize}

For each of the models, different optimization problems need to be solved to calculate the robustness of their predictions.
We now list all of these optimization problems and give for each of the three types of perturbations the formulas needed to solve these problems.
Even though there has already been work on calculating the robustness of predictions of NBCs and GeFs w.r.t. \(\varepsilon\)-contamination, we nevertheless give the formulas to make the paper self-contained and to make it easier to compare with the other perturbations.
For previous work on the robustness of NBCs w.r.t. \(\varepsilon\)-contamination, we refer to the work of \citet{detavernier2025robustness} and for GeFs to the work of \citet{correia2020robustclassificationdeepgenerative}.
The formulas for the TV distance and \(\chi^2\)-divergence balls are particular cases of more general formulas in recent work of \citet{lowerexpectationsref}, which simplify in our particular context.

Firstly, we know from Theorem~\ref{th:NBC_robustness} that for the NBC we need to calculate the lower and upper probabilities of \(\proba_{\Feature{i}\vert\prediction}(\feature{i})\) w.r.t. the perturbations \(\featprobfamily^\delta\), for all \(\class \in \Classes\) and \(i \in \{1, \dots, \numfeatures\}\).
These are particular instances of the task of computing the lower and upper probability of an outcome over a perturbation \(\mathcal{Q}^{\delta}\) of a mass function \(q\) on \(\Omega\), which we denote by \(\underline{q}_{\mathcal{Q}^{\delta}}(\omega) = \min_{q' \in \mathcal{Q}^{\delta}} q'(\omega)\) and \(\overline{q}_{\mathcal{Q}^{\delta}}(\omega) = \max_{q' \in \mathcal{Q}^{\delta}} q'(\omega)\) for any \(\omega \in \Omega\).
For our three perturbations the formulas for these bounds are given in \tableref{tab:lower_upper_probabilities}.
\begin{table}[htbp]
  \centering
  \caption{The formulas for \(\underline{q}_{\mathcal{Q}^{\delta}}(\omega)\) (left) and \(\overline{q}_{\mathcal{Q}^{\delta}}(\omega)\) (right) for our three perturbations.}
  \label{tab:lower_upper_probabilities}
  \begin{tabular}{c||l|l}
    \toprule
    \(\varepsilon\) & \((1-\delta)q(\omega)\) & \((1-\delta)q(\omega) + \delta\) \\
    TV & \(\max\{0, q(\omega) - \delta\}\) & \(\min\{1, q(\omega) + \delta\}\) \\
    \(\chi^2\) & \(\max\{0, q(\omega) - \sqrt{\delta q(\omega)(1-q(\omega))}\}\) & \(\min\{1, q(\omega) + \sqrt{\delta q(\omega)(1-q(\omega))}\}\) \\
    \bottomrule
  \end{tabular}
\end{table}


Next, still for the NBCs, by Theorem~\ref{th:NBC_robustness} we also need to calculate the lower expectation of \(\indicator{\prediction}\alpha - \indicator{\class}\beta\) w.r.t. the local perturbations \(\classprobfamily^{\delta}\) of \(\proba_{\Class}\), for all \(\class \in \Classes \backslash \{\prediction\}\), where \(\prediction\) is the predicted class by the NBC and with \(\alpha, \beta\geq0\).
In the equation in Theorem~\ref{th:NBC_robustness}, \(\alpha\) and \(\beta\) represent products of lower and upper probabilities that do not depend on the perturbation \(\classprobfamily^{\delta}\), which is why we write them as constants here.
The formulas for \(\lowprev{\parampert{\delta}}{\indicator{\prediction}\alpha - \indicator{\class}\beta}\) for our three perturbations are given in \tableref{tab:lower_expectation_f},
where \(\mu\) and \(\sigma^2\) are the expected value and variance of \(\indicator{\prediction}\alpha - \indicator{\class}\beta\) under \(\proba_{\Class}\), and \(\mu_2\) and \(\sigma_2^2\) are the expected value and variance of the restriction of this function to \(\Classes \backslash \{\prediction\}\) under the mass function \(\proba_{\Class}^*\) on \(\Classes \backslash \{\prediction\}\) defined by \(\proba_{\Class}^*(\class) = \frac{\proba_{\Class}(\class)}{1-\proba_{\Class}(\prediction)}\) for all \(\class \in \Classes \backslash \{\prediction\}\).
\begin{table}[htbp]
  \centering
  \caption{The formulas for \(\lowprev{\classprobfamily^{\delta}}{\indicator{\prediction}\alpha - \indicator{\class}\beta}\) for our three perturbations.}
  \label{tab:lower_expectation_f}
  \begin{tabular}{c||l}
    \toprule
    \(\varepsilon\) & \((1-\delta)(\alpha \proba_{\Class}(\prediction) - \beta \proba_{\Class}(\class)) - \delta\beta\) \\
    TV & \(\alpha\max\{0, \proba_{\Class}(\prediction) - \delta\} - \beta\min\{1, \proba_{\Class}(\class) + \delta\}\) \\
    \(\chi^2\) & \( \begin{cases}
                        \mu - \sigma\sqrt{\delta} & \text{if } \delta < \frac{\sigma^2}{(\alpha-\mu)^2}\\
                        -\beta & \text{if } \delta \geq \frac{1}{1-\proba_{\Class}(\prediction)}\left( \frac{\sigma^2_2}{\mu^2_2}+\proba_{\Classes}(\prediction) \right)\\
                        \mu_2 - \sigma_2\sqrt{\delta(1-\proba_{\Class}(\prediction)) - \proba_{\Class}(\prediction)} & \text{otherwise}
                    \end{cases}\) \\
    \bottomrule
  \end{tabular}
\end{table}

Lastly, for the GeFs, given a perturbation \(\pertweightsvec{\node}^{\delta}\) of the weights of a sum node \(\node\), we need to compute the minimum weighted average of the values of its children, where the value of a child \(\childnode \in \childrennode{\node}\) is given by the function \(\vfunc{\childnode}\).
In the GeFs we consider, this average always consists of only two terms, since all sum nodes of which we perturb the weights have only two children.
If we let \(\childnode_1\) and \(\childnode_2\) be the two children of \(\node\), such that, without loss of generality, \(\vfunc{\childnode_1} \leq \vfunc{\childnode_2}\), then this minimum reduces to:
\begin{equation}
    \min_{\weightsnode_{\node}' \in \pertweightsvec{\node}^{\delta}} \weightsnode_{\node, \childnode_1}' \vfunc{\childnode_1} + \weightsnode_{\node, \childnode_2}' \vfunc{\childnode_2} = \vfunc{\childnode_1} + (\vfunc{\childnode_2} - \vfunc{\childnode_1})\underline{\weightsnode}_{\node,\childnode_2},
\end{equation}
where \(\underline{\weightsnode}_{\node,\childnode_2} = \min_{\weightsnode_{\node}' \in \pertweightsvec{\node}^{\delta}} \weightsnode_{\node, \childnode_2}'\) is the minimum of \(\weightsnode_{\node, \childnode_2}\) w.r.t. the perturbation \(\pertweightsvec{\node}^{\delta}\).
Since the weight vector of a node is a mass function, finding this is equivalent to calculating a lower probability.
We can thus use the formulas in the left column of \tableref{tab:lower_upper_probabilities} to calculate this minimum for each of our three perturbations.

\section{Experiments}\label{sec:experiments}
We present some preliminary experiments\footnote{The source code is available at \url{https://github.com/addtaver/RobustnessQuantification}} on 11 datasets to illustrate the potential of the robustness values defined in this paper.
In the following, we denote the robustness values for \(\varepsilon\)-contamination, TV distance balls and \(\chi^2\)-divergence balls as \(\rob_{\varepsilon}\), \(\rob_{\mathrm{TV}}\) and \(\rob_{\chi^2}\) respectively.
For details on the experimental setup, training procedure and the datasets we used, we refer to Appendix~\ref{ap:experiments}.
We also compare our three robustness values with another approach that is commonly used to assess the trustworthiness of predictions, namely Uncertainty Quantification (UQ).
Details on the six UQ metrics we compare with (\(u_{\mathrm{marg}}\), \(u_{\mathrm{max}}\), \(u_{\mathrm{H}}\), \(u_{\mathrm{t}}\), \(u_{\mathrm{a}}\), \(u_{\mathrm{e}}\)) are given in Appendix~\ref{ap:uq_metrics}.

To evaluate the performance of our approach, we check how the accuracy of the model changes as we continuously reject the predictions with the lowest robustness (or highest uncertainty); the idea being that if robustness is a good indicator of the trustworthiness of a prediction, then the accuracy on the remaining predictions should increase as we reject more and more instances.
A curve showing for every percentage of rejected instances the accuracy on the remaining predictions, is called an \emph{accuracy-rejection curve} (ARC); \figureref{fig:acc_rej_curve} shows an example, where each line represents an average over 10 different runs.
\begin{figure}[htbp]
    \centering
    \includegraphics[width=\textwidth, trim=0cm .5cm 0cm 0.3cm]{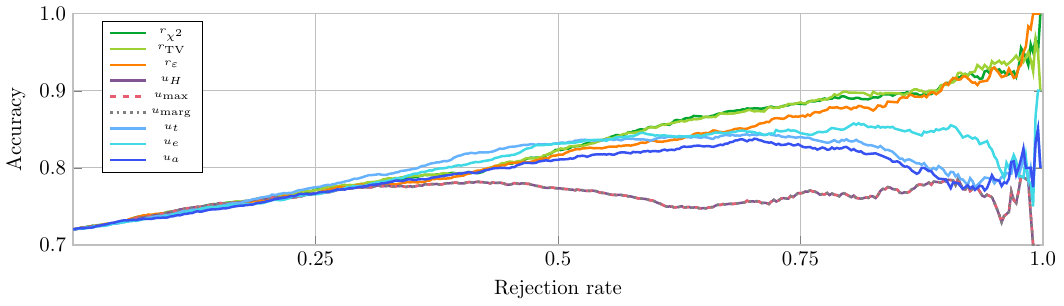}
    \caption{Accuracy-rejection curve for dataset D5 for GeFs.}
    \label{fig:acc_rej_curve}
\end{figure}
On this particular dataset, the three robustness measures seem to perform best overall since they keep increasing until the very end, more so than the other metrics shown.
However, if we would focus on specific Rejection Rates (RRs), then other metrics sometimes perform better.
Take for example \(u_{\mathrm{t}}\), which is clearly the best measure for RRs of around 40\%.

Focussing on specific RRs can be very relevant in practice if you only want to let the model make a prediction for the most trustworthy \(X\%\) of instances.
To evaluate the performance of the robustness values in such use cases, we focus on specific RRs in the ARC and check what metric of trustworthiness performs best for each of these RRs.
The results are summarized in \tableref{tab1} and \tableref{tab2}, where for each specific RR and trustworthiness metric the mean value of the accuracy at that RR over all datasets is shown, with the number of wins (W) for that metric at that RR next to it.
A win is defined as having the highest accuracy at that RR for a given dataset, and ties are solved by giving a win to all metrics that are tied for the highest accuracy.
We also included in the last row of these tables the average area under the ARC (AU-ARC) for each metric over all datasets, which instead of focussing on specific RRs, gives an overall indication of the performance of a metric as an indicator of trustworthiness.
The AU-ARC can also help us to summarize and compare ARCs in a more quantitative way, since comparing ARCs by eye can be very difficult, especially when there are many of them.

\begin{table}[htbp]
  \centering
  \caption{Each row, except the last one containing the AU-ARC, shows the mean accuracy on the remaining instances and the number of datasets for which each metric had the highest accuracy at that RR for the NBC.}
  \label{tab1}
  \includegraphics[width=\textwidth, trim=0cm .3cm 0cm 0.3cm]{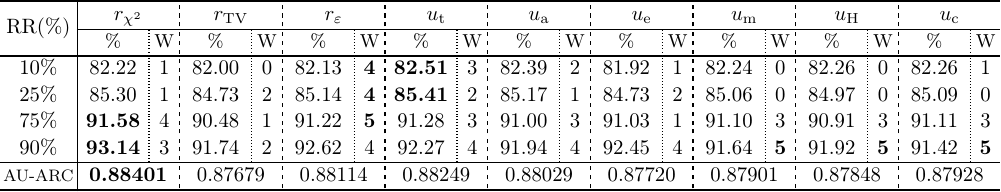}
\end{table}

\begin{table}[htbp]
  \centering
  \caption{Each row, except the last one containing the AU-ARC, shows the mean accuracy on the remaining instances and the number of datasets for which each metric had the highest accuracy at that RR for the GeF.}
  \label{tab2}
  \includegraphics[width=\textwidth, trim=0cm .3cm 0cm 0.3cm]{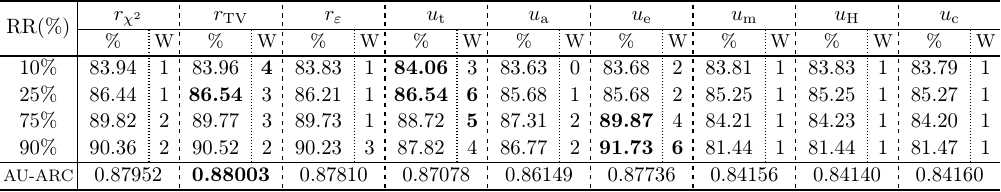}
\end{table}

If we focus on the two new robustness values in these tables, we see that \(\rob_{\chi^2}\) seems to be the best performing metric for the NBC, especially for higher RRs, while \(\rob_{\mathrm{TV}}\) seems to be the best performing metric for the GeF, especially for lower RRs.
The robustness value \(\rob_{\varepsilon}\) has similar AU-ARCs as the others, seems to do very well in terms of the number of wins for the NBC, and is just lacking behind the other two for the GeF.
For the UQ measures, we see that \(u_{\mathrm{e}}\) and \(u_{\mathrm{t}}\) seem to perform very well, since they tend to have a high number of wins and the highest AU-ARCs for certain RRs.
Overall, since some of the differences in performance are rather small, we prefer to refrain from drawing strong conclusions about which is the `best' metric, especially so since we observe that the relative performance of the different metrics depends on the model, RR and performance metric
We can conclude however that all three robustness values are competitive as indicators of trustworthiness.
What makes this particularly interesting, is that the metrics \(u_{\mathrm{t}}\), \(u_{\mathrm{a}}\) and \(u_{\mathrm{e}}\) use an ensemble that requires retraining the classifier on bootstrap samples of the training data.
This retraining can be computationally expensive though and is not always possible. This makes our robustness values, which do not require retraining, a promising alternative.


\section{Conclusion and Future Work}
The main contribution of this paper is that we generalized existing RQ approaches for NBCs and GeFs, which only considered \(\varepsilon\)-contamination, to a more general approach that allows for any type of perturbation of the local parameters of these models.
To illustrate this, we applied our results to two specific types of perturbations, namely TV distance and \(\chi^2\)-divergence balls, and provide the formulas needed to calculate the robustness of predictions w.r.t. these perturbations (and, for the sake of completeness, also for \(\varepsilon\)-contamination).
Finally, our experiments demonstrate that our new robustness values w.r.t. TV distance and \(\chi^2\)-divergence balls can be used as indicators for the trustworthiness of predictions.

Possible extensions of this work could be to consider other types of perturbations of the local parameters, or other types of models.
Since the behavior of the robustness values seems to be quite different for different contexts and models, it would also be interesting to study what perturbations are best suited for which contexts and models.
A start could be to study and compare our new methods in a context with noise or distribution shift, as previous work has shown that RQ is particularly well-suited in such contexts in comparison with UQ \citep{detavernier2025robustness,detavernier2026robustnessquantificationuncertaintyquantification}.

Previous work \citep{detavernier2026robustnessquantificationuncertaintyquantification} also shows that RQ and UQ can be complementary, and can be usefully combined to obtain even better assessments of trustworthiness. It would therefore be interesting to employ our new RQ metrics in such combinations as well.

\section*{Acknowledgments}
We would like to thank the anonymous reviewers for their time, kind words and helpful feedback.
The work of both authors was partially supported by Ghent University's Special Research Fund, through Jasper De Bock's Basic Research Funding project entitled ``Modelling Uncertainty with Imprecise Probabilities''.

\bibliography{references}

\appendix
\section{Computing Lower Expectations for GeFs}\label{app:proofs}
In the paper we consider perturbations \(\pertweightsvec{\node}^{\delta}\) of the weights \(\weightsnode_{\node}\) of a sum node \(\node\) in a GeF, where we implicitly assume that they are part of parametrized perturbations \(\pertweightsvec{\node}^{\Delta}\), with \(\delta\) taking values in \(\Delta\).
For notational convenience, in the following result we will simply consider any perturbation \(\pertweightsvec{\node}\) of the weights of a sum node \(\node\).
For the sake of generality, we also consider a general perturbation \(\pertweightsvec{\rootnode}\) of \(\weightsnode_{\rootnode}\) for the root node \(\rootnode\) of the GeF, rather than consider the trivial case \(\pertweightsvec{\rootnode} = \{\weightsnode_{\rootnode}\}\) as in the paper.
We collect all possible combinations of perturbations of the weights of the sum nodes in the tree rooted at \(\node\) in \(\pertweightstree{\node}\coloneq \pertweightsvec{\node} \times \bigtimes_{\childnode \in \descnoleafnode{\node}} \pertweightsvec{\childnode}\), and we denote the corresponding perturbation of the joint mass function \(\proba_{\weightstree_{\node}}\) by \(\probfamily_{\pertweightstree{\node}}\).
\begin{theorem}
    \label{th:GeDT_lower}
    Consider a GeF rooted at a node \(\rootnode\) with for each sum node \(\node\) a weight vector \(\weightsnode_{\node}\) and a perturbation \(\pertweightsvec{\node}\) of \(\weightsnode_{\node}\).
    For any (not necessarily sum) node \(\node\), let \(\probfamily_{\pertweightstree{\node}}\) be the corresponding perturbation of \(\proba_{\weightstree_{\node}}\).
    Then for any function \(g:(\Classes,\FeaturesSet)\to\mathbb{R}\):
    \begin{equation}
        \vfunc{\node} = \lowprev{\probfamily_{\pertweightstree{\node}}}{g},
    \end{equation}
    where, the function \(\underline{V}\) node is defined recursively by
    \begin{equation}
        \vfunc{\node} \coloneq \begin{cases}
            E_{\proba_{\node}}(g) & \text{if } \node \text{ is a leaf node},\\
            \min_{\weightsnode_{\node}' \in \pertweightsvec{\node}} \sum_{\childnode \in \childrennode{\node}} \weightsnode_{\node, \childnode}' \vfunc{\childnode} & \text{if } \node \text{ is a sum node}.
        \end{cases}
    \end{equation}
\end{theorem}

For our proof of this result, we took inspiration from the work of \citet{DERATANIMAUA2018163}, who define a recursive function on the nodes of a sum-product network to calculate conditional lower expectations.
Their result is on the one hand more general because it applies to sum-product networks that are more general than GeFs, but on the other hand less general because it only applies to specific inferences.
Their focus is also on the complexity of computations, whereas our focus is on calculating the lower expectation.\\

\begin{proof}\textbf{of Theorem~\ref{th:GeDT_lower}}
    First, consider the case where \(\node\) is a leaf node.
    In this case, the function \(\vfunc{\node}\) is given by \(\vfunc{\node} = E_{\proba_{\node}}(g)\).
    Since \(\node\) is a leaf node, there are no sum nodes in the tree rooted at \(\node\) and thus no weights to perturb, so \(\probfamily_{\pertweightstree{\node}} = \{\proba_{\node}\}\).
    Hence, the lower expectation of \(g\) over \(\probfamily_{\pertweightstree{\node}}\) is given by \(\lowprev{\probfamily_{\pertweightstree{\node}}}{g} = \min_{\proba \in \probfamily_{\pertweightstree{\node}}} E_{\proba}(g) = E_{\proba_{\node}}(g)\), as claimed.

    For the case where \(\node\) is a sum node, we prove the claim by induction on the structure of the tree rooted at \(\node\).
    Assuming that the claim holds for all children of \(\node\), we will show that it also holds for \(\node\).
    Since
    \(
        \vfunc{\node} = \min_{\weightsnode_{\node}' \in \pertweightsvec{\node}} \sum_{\childnode \in \childrennode{\node}} \weightsnode_{\node, \childnode}' \vfunc{\childnode},
    \)
    it follows from the induction hypothesis, that
    \begin{equation}
        \vfunc{\node} = \min_{\weightsnode_{\node}' \in \pertweightsvec{\node}} \sum_{\childnode \in \childrennode{\node}} \weightsnode_{\node, \childnode}' \lowprev{\probfamily_{\pertweightstree{\childnode}}}{g} = \min_{\weightsnode_{\node}' \in \pertweightsvec{\node}} \sum_{\childnode \in \childrennode{\node}} \weightsnode_{\node, \childnode}' \left( \min_{\proba \in \probfamily_{\pertweightstree{\childnode}}} \prev{\proba}{g} \right).
    \end{equation}
    Since each set of weights \(\weightstree_{\childnode}' \in \pertweightstree{\childnode}\) corresponds to an element \(\proba_{\weightstree_{\childnode}}'\) of \(\probfamily_{\pertweightstree{\childnode}}\), each minimization inside the sum is equivalent to minimizing over the perturbed sets of weights.
    By applying this, and the definition of the expectation, we get
    \begin{equation}
        \vfunc{\node} = \min_{\weightsnode_{\node}' \in \pertweightsvec{\node}} \sum_{\childnode \in \childrennode{\node}} \weightsnode_{\node, \childnode}' \left( \min_{\weightstree_{\childnode}' \in \pertweightstree{\childnode}} \sum_{(\class, \features)\in\Classes\times\Features} \proba_{\weightstree_{\childnode}'}(\class, \features) g(\class, \features) \right).
    \end{equation}
    Since the weights of a sum node are non-negative, we can take for each of the children of \(\node\) the minimum out of the sum:

    \begin{equation}
        \vfunc{\node} = \min_{\weightsnode_{\node}' \in \pertweightsvec{\node}} \min_{(\childnode \in \childrennode{\node})\weightstree_{\childnode}' \in \pertweightstree{\childnode}} \sum_{(\class, \features)\in\Classes\times\Features} \left( \sum_{\childnode \in \childrennode{\node}} \weightsnode_{\node, \childnode}'   \proba_{\weightstree_{\node}'}(\class, \features) \right) g(\class, \features).
    \end{equation}
    Since all perturbations of the weights are independent of each other, minimizing over all these perturbations separately is equivalent to minimizing over all perturbed sets of weights \(\weightstree_{\node}' \in \pertweightstree{\node}\) in the tree rooted at \(\node\).
    Hence, we can rewrite the above as
    \begin{equation}
        \vfunc{\node} = \min_{\weightstree_{\node}' \in \pertweightstree{\node}} \sum_{(\class, \features)\in\Classes\times\Features}  \left( \sum_{\childnode \in \childrennode{\node}} \weightsnode_{\node, \childnode}'   \proba_{\weightstree_{\childnode}}(\class, \features) \right) g(\class, \features).
    \end{equation}
    Finally, if we use the fact that \(\proba_{\weightsnode_{\node}'}(\class, \features) = \sum_{\childnode \in \childrennode{\node}} \weightsnode_{\node, \childnode}' \proba_{\weightstree_{\childnode}'}(\class, \features)\), and reuse similar arguments as before, we can finally rewrite the above as
    \begin{equation}
        \vfunc{\node} = \min_{\weightstree_{\node}' \in \pertweightstree{\node}} \sum_{(\class, \features)\in\Classes\times\Features} \proba_{\weightstree_{\node}'}(\class, \features) g(\class, \features) = \min_{\proba \in \probfamily_{\pertweightstree{\node}}} E_{\proba}(g) = \lowprev{\probfamily_{\pertweightstree{\node}}}{g},
    \end{equation}
    as claimed.
\end{proof}

\section{Additional Details on the Experiments}\label{ap:experiments}
The experiments we presented in Section~\ref{sec:experiments} were conducted on several datasets from the UCI Machine Learning Repository \citep{ucimlrepository}.
The datasets we used are shown in \tableref{tab:datasets}.
Next we give some details on how we cleaned these datasets and adapted the tasks to fit our setting.
First, we removed all continuous features from the datasets, because we focus on discrete features in this work.
Then we removed all instances with missing values from the datasets.
For some of the datasets, we performed more specific cleaning.
For the Solar Flare datasets and the Student Performance datasets, we adapted the task to one that leans more toward standard classification.
The former datasets are originally about predicting the number of solar flares that occur, and this for three types of flares; we made this classification task binary, with the aim to predict whether a solar flare occurs (of any type) or not.
For the Student Performance datasets,  the original task is to predict the grade of a student, but we adapted this to a binary classification task as well, where we predict whether the student passes or fails.

\renewcommand{\arraystretch}{1.0}
\begin{table}[htbp]
  \centering
  \caption{Details of the datasets used in the experiments.}
  \label{tab:datasets}
    \begin{tabular}{llccc}
      \toprule
      ID & Dataset & Size & \(\vert\Classes\vert\) & \(\vert\FeaturesSet\vert\)\\
      \midrule
      D1 & Adult & 3000 & 2 & 7\\
      D2 & Australian Credit Approval & 690 & 2 & 8\\
      D3 & Bank Marketing & 3000 & 2 & 9\\
      D4 & Breast Cancer Wisconsin & 683 & 2 & 9\\
      D5 & German Credit Data & 1000 & 2 & 13\\
      D6 & National Poll Healthy Aging & 714 & 3 & 13\\
      D7 & Solar Flare (big) & 1066 & 2 & 10\\
      D8 & Solar Flare (small) & 323 & 2 & 10\\
      D9 & SPECT Heart & 428 & 2 & 22\\
      D10 & Student Performance Math & 649 & 2 & 29\\
      D11 & Student Performance Port & 649 & 2 & 29\\
      \bottomrule
    \end{tabular}
\end{table}

The training procedure for the models is as follows.
Unless the dataset comes with a predefined split, we split the datasets into a training and test set, where the test set is 40\% of the whole dataset, and the size of the test set is maxed out at 500 instances.
The NBC is trained by first optimizing a smoothing parameter with 5-fold cross-validation on the training set and then training the model on the whole training set with this smoothing parameter.
For the details on how the GeFs are trained, we refer to the original paper of \citet{NEURIPS2020_8396b14c}, where we choose to combine 25 GeDTs and where the distributions of the leaf nodes are learned using the maximum likelihood principle.
The whole training procedure, including the splitting of the datasets, is repeated 10 times to reduce the effect of randomness on the results.
The ARCs are then calculated for each of the 10 runs and averaged out.

\section{Uncertainty Quantification Metrics}\label{ap:uq_metrics}
Another approach to assess the trustworthiness of a prediction is to try to quantify the uncertainty of the prediction, where (unlike for robustness) a higher uncertainty corresponds to a less trustworthy prediction.

In our experiments, we compare our robustness values with six commonly used uncertainty quantification measures.
The three straightforward measures are one minus the (conditional) predicted probability of the predicted class, denoted as \(u_{\mathrm{m}}\), the entropy of the predicted (conditional) distribution over the classes, denoted as \(u_{\mathrm{H}}\), and the margin of confidence being the difference between the highest and second highest predicted (conditional) probabilities, denoted as \(u_{\mathrm{c}}\).
Besides these, we also use some more complex measures that are based on the information-theoretic decomposition of uncertainty as is explained in the work of \citet{hullerenwillem}.
To estimate these, an ensemble of 10 models is trained on a bootstrap sample of the training set of the same size as the original training set \citep{pmlr-v80-depeweg18a,Mobiny2021,Shaker2020}.
The so-called the \emph{total uncertainty} \(u_{\mathrm{t}}\) is estimated as the entropy of the average predicted (conditional) distribution over all models, \emph{aleatoric uncertainty} \(u_{\mathrm{a}}\) is then estimated as the average entropy of the predicted (conditional) distributions over all models and the \emph{epistemic uncertainty} \(u_{\mathrm{e}}\) is estimated as the difference between the total and aleatoric uncertainty.

For these uncertainty measures, the corresponding ARCs in our experiments are obtained by first rejecting instances with the highest uncertainty.

\end{document}